\documentclass[11pt,journal,draftcls,onecolumn,peerreviewca]{IEEEtran}

\usepackage{amsfonts}
\usepackage[cmex10]{amsmath}
\usepackage{amsthm}
\usepackage{cite}
\usepackage{color,soul}
\usepackage[pdftex]{graphicx}
\usepackage{bm}
\usepackage{booktabs}
\usepackage{algorithm}
\usepackage{algorithmic}
\usepackage{mdwmath}
\usepackage{mdwtab}
\usepackage{multirow}
\usepackage{diagbox}

\newtheorem{theorem}{Theorem}

\newtheorem{definition}{Definition}
\newtheorem{remark}{Remark}

\newcommand{\reals}{{\mathbb R}}
\newcommand{\prox}{\mathrm{prox}}
\newcommand{\argmin}{\mathop{\rm argmin}}

\begin{document}
%\begin{CJK*}{GBK}{fs}

\title{Nonconvex Sparse Logistic Regression with Weakly Convex Regularization}

\author{Xinyue Shen and Yuantao Gu
\thanks{The authors are with Department of 
	Electronic Engineering and Tsinghua National Laboratory for Information Science and Technology (TNList), 
	Tsinghua University, Beijing 100084, China (e-mail: gyt@tsinghua.edu.cn).}}

\maketitle

\begin{abstract}
In this work we propose to fit a sparse logistic regression model
by a weakly convex regularized nonconvex optimization problem.
The idea is based on the finding that a weakly convex function
as an approximation of the $\ell_0$ pseudo norm 
is able to better induce sparsity than the commonly used $\ell_1$ norm.
For a class of weakly convex sparsity inducing functions,
we prove the nonconvexity of the corresponding sparse logistic regression problem,
and study its local optimality conditions and the choice of the regularization parameter 
to exclude trivial solutions.
Despite the nonconvexity,
a method based on proximal gradient descent
is used to solve the general weakly convex sparse logistic regression, 
and its convergence behavior is studied theoretically.
Then the general framework is applied to a specific weakly convex function, 
and a necessary and sufficient local optimality condition is provided.
The solution method is instantiated in this case as an iterative firm-shrinkage algorithm,
and its effectiveness is demonstrated in numerical experiments
by both randomly generated and real datasets.
\end{abstract}


\section{Introduction}
Logistic regression is a widely used supervised machine learning
method for classification.
It learns a neutral hyperplane
in the feature space of a learning problem
according to a probabilistic model,
and classifies test data points accordingly.
The output of the classification result does not only give
a class label, but also a natural probabilistic interpretation.
It can be straightforwardly extended from two-class to multi-class problems,
and it has been applied to text classification \cite{Genkin2006},
gene selection and microarray analysis \cite{Zhu2004,Cawley2006},
combinatorial chemistry \cite{Cronin2002}, image analysis \cite{Murray2011, Ciocca2015}, etc.

In a classification problem $N$ pairs of training data
$\{({\bf x}^{(i)},y^{(i)}), i=1,\ldots,N \}$ are given,
where every point ${\bf x}^{(i)}\in\reals^d$ is a feature vector in the $d$
dimensional feature space, 
and $y^{(i)}$ is its corresponding class label. 
In a two-class logistic regression problem, 
$y^{(i)}\in\{0,1\}$, and it is assumed that the probability distribution of a class label
$y$ given a feature vector ${\bf x}$  is as the following
\begin{align}\label{label_prob}
p(y=1|{\bf x}; \bm{\theta}) &= \sigma(\bm{\theta}^{\rm T} {\bf x}) = \frac{1}{1+\exp (-\bm{\theta}^{\rm T} {\bf x})} \nonumber\\
p(y=0|{\bf x}; \bm{\theta}) &= 1 - \sigma(\bm{\theta}^{\rm T} {\bf x}) 
 = \frac{1}{1+\exp (\bm{\theta}^{\rm T} {\bf x})},
\end{align}
where $\sigma(\cdot)$ is the sigmoid function defined as above,
and $\bm{\theta}\in\reals^d$ is the model parameter to be learned.
When ${\bf x}^{\rm T}\bm{\theta} = 0$, the probability of having either label is $0.5$,
and thus the vector $\bm{\theta}$ gives the normal vector of a
neutral hyperplane.
Notice that if an affine hyperplane $\bm{\theta}^{\rm T} {\bf x}+b$ is to be considered,
then we can simply add an additional dimension with value $1$ to every feature vector, and then it will have the linear hyperplane form.

Suppose that the labels of the training samples are independently drawn
from the probability distribution \eqref{label_prob},
then it has been proposed to learn $\bm{\theta}$
by minimizing the negative log-likelihood
function, and the optimization problem is as follows
\begin{equation}\label{log_reg}
\begin{array}{ll}
\mbox{minimize} \quad l(\bm{\theta}),
\end{array}
\end{equation}
where $\bm{\theta}\in\reals^d$ is the variable,
and $l$ is the (empirical) logistic loss
\begin{equation}\label{logistic_loss}
l(\bm{\theta}) = \sum_{i=1}^N -\log p(y^{(i)} | {\bf x}^{(i)}; \bm{\theta}).
\end{equation}
Problem \eqref{log_reg} is convex and differentiable, and can be
readily solved \cite{hastie01}.
Once we obtain a solution $\hat{\bm{\theta}}$,
given a new feature vector $\bf x$, we can predict
the probability of the two possible labels according to the logistic model,
and take the one with larger probability by
\begin{align*}
y = {\bf 1}\left({\bf x}^{\rm T}\hat{\bm{\theta}} \geq 0 \right)
= \left\{
\begin{array}{ll}
1, \quad &{\bf x}^{\rm T}\hat{\bm{\theta}} \geq 0; \\
0, \quad&{\bf x}^{\rm T}\hat{\bm{\theta}} < 0.
\end{array}
\right.
\end{align*}

When the number of training samples $N$ 
is relatively small compared to the feature space dimension $d$,
adding a regularization can avoid over-fitting and enhance classification accuracy on test data, 
and the $\ell_2$ norm has long been used as a regularization function
\cite{Chaloner1989,Komarek2004,Minka2007}.
Furthermore, a sparsity-inducing regularizer can select a subset of all available features
that capture the relevant properties.
Since $\ell_1$ norm is a convex function that induces sparsity,
 the $\ell_1$ norm regularized sparse logistic regression prevails
\cite{Roth2002,Shevade2003,Lee2006,Genkin2006,Koh2007,plan2013}.

Despite that in general nonconvex optimization is hard to solve globally,
nonconvex regularization has been extensively studied to induce sparsity
in sparse logistic regression \cite{Loh2013,Yang2016} and other sparsity related topics
such as compressed sensing \cite{Chartrand2007,Chen2014,zhu2015}.
Inspired by results that tie binomial regression and one-bit compressed
sensing \cite{plan2013}, as well as results in compressed sensing indicating that 
weakly convex functions are able to better induce sparsity than 
the $\ell_1$ norm \cite{Chartrand2007,Chen2014,Shen2016},
in this work we propose to use a weakly convex function
in sparse logistic regression.

\subsection{Contribution and outline}
In this work, we consider a logistic regression problem in which 
the model parameter $\bm{\theta}$ is sparse,
i.e., the dimension $d$ can be large, and $\bm{\theta}$ is assumed to have
only $K$ non-zero elements, where $K$ is relatively small compared to $d$.
We propose the following problem that uses a weakly convex (nonconvex) function $J$ 
in sparse logistic regression
\begin{equation}\label{main}
\begin{array}{ll}
\mbox{minimize} \quad l(\bm{\theta}) + \beta J(\bm{\theta}),
\end{array}
\end{equation}
where the variable is $\bm{\theta}\in\reals^d$, 
$\beta>0$ is a regularization parameter, and $l$ is the logistic loss \eqref{logistic_loss}.
The contribution of this work can be summarized as the following.
\begin{itemize}
	\item We introduce weakly convex (nonconvex) sparsity inducing 
	functions into sparse logistic regression.
	We prove that the general weakly convex regularized optimization problem \eqref{main} is nonconvex.
	Its local optimality conditions are studied,
	as well as the range of the regularization parameter $\beta$ 
	to exclude $\bf 0$ as a local optimum.
	These will be in section \ref{sec:sparse_log_reg}.
	
	\item A solution method based on proximal gradient is proposed
	to solve the general problem of weakly convex regularized 
	sparse logistic regression \eqref{main}.
	Despite its nonconvexity,
	we provide a conclusion on the convergence behavior, which shows that
	the objective function is able to monotonically decrease and converge.
	These will be in section \ref{sec:solution_method}.
	
	\item We apply the general framework to a specific weakly convex regularizer.
	A necessary and sufficient condition on its local optimality is obtained, and the convergence analysis of the solution method for the general problem can also be applied. 
	These will be in section \ref{sec:specific}.
	In numerical experiments in section \ref{sec:exp}, we use this specific choice of function
	to verify the effectiveness of the model and the method on both randomly generated and real datasets.
\end{itemize}

\subsection{Notations}
In this work, for a vector $\bf x$, we use $\|{\bf x}\|_2$ to denote its $\ell_2$
norm, $\|{\bf x}\|_1$ to denote its $\ell_1$ norm, and $\|{\bf x}\|_\infty$ to denote
its infinity norm.
Its $i$th entry is denoted as ${\bf x}_i$.
For a matrix ${\bf X}$,
$\|{\bf X}\|$ is its operator norm, i.e., its largest singular value.
For a differentiable function $f: \reals^n \rightarrow \reals$,
its gradient is denoted as $\nabla f$, and if it is twice differentiable,
then its Hessian is denoted as $\nabla^2 f$. If $f$ is a convex function,
$\partial f({\bf x})$ is its subgradient set at point $\bf x$.
For a function $F:\reals\rightarrow \reals$, we use $F^\prime_-$
and $F^\prime_+$ to denote its left and right derivatives, and
$F^{\prime\prime}_-$ and $F^{\prime\prime}_+$ to denote its left and right second derivatives, if they exist.

\section{Related works}
\label{sec:related}
\subsection{$\ell_2$ and $\ell_1$ regularized logistic regression}

The $\ell_2$ regularized logistic regression problem
is as the following
\begin{equation*}
\begin{array}{ll}
\mbox{minimize} \quad l(\bm{\theta}) +\beta \|\bm{\theta}\|_2^2,
\end{array}
\end{equation*}
where $\bm{\theta}$ is the variable, $l$ is the logistic loss \eqref{logistic_loss},
and $\beta>0$ is the regularization parameter.
The solution can be interpreted as the maximum a posteriori probability (MAP) estimate
of $\bm{\theta}$, if $\bm{\theta}$ has a Gaussian prior distribution with zero mean 
and covariance $\beta {\bf I}$  \cite{Chaloner1989}.
The problem is strongly convex and differentiable,
and can be solved by methods such as the Newton, quasi-Newton, coordinate descent,
conjugate gradient descent, and iteratively reweighted least squares.
For example see \cite{Komarek2004,Minka2007} and references therein.

It has been known that minimizing the $\ell_1$ norm of a variable
induces sparsity to its solution,
so the following $\ell_1$ norm regularized sparse logistic regression has been widely
used to promote the sparsity of $\bm{\theta}$
\begin{equation}\label{prob:l1}
\begin{array}{ll}
\mbox{minimize} \quad  l(\bm{\theta}) + \beta \|\bm{\theta}\|_1,
\end{array}
\end{equation}
where $\bm{\theta}$ is the variable, 
$l$ is the logistic loss \eqref{logistic_loss},
and $\beta>0$ is a parameter balancing the sparsity
and the classification error on the training data.
In logistic regression, $\bm{\theta}_j = 0$ means that the $j$th feature
does not have influence on the classification result.
Thus, sparse logistic regression tries to find a few features that
are relevant to the classification results from a large number of features.
Its solution can also be interpreted as an MAP estimate,
when $\bm{\theta}$ has a Laplacian prior distribution
$p(\bm{\theta}) = (\beta/2)^d \exp (-\beta \|\bm{\theta}\|_1)$.

The problem \eqref{prob:l1} is convex but nondifferentiable, 
and several specialized solution methods have been proposed, such as an
iteratively reweighted least squares (IRLS) method \cite{Lee2006}
in which every iteration solves a LASSO \cite{Tibshirani1996},
a generalized LASSO method \cite{Roth2002},
a coordinate descent method \cite{Genkin2006}, 
a Gauss-Seidel method \cite{Shevade2003},
an interior point method that scales well to large problems \cite{Koh2007},
and some online algorithms such as \cite{Perkins03}.

\subsection{Nonconvex sparse logistic regression and SVMs}

The work \cite{Loh2013} studies properties of local optima
of a class of nonconvex regularized M-estimators
including logistic regression and the convergence behavior
of a proposed composite gradient descent solution method. 
The nonconvex regularizers considered
in their work overlap with the ones in this work,
but they have a convex constraint in addition.

Difference of convex (DC) functions are proposed
in works such as \cite{LeThi2008,Cheng2013,Yang2016}
 to approximate the $\ell_0$ pseudo norm and work as the regularization 
for feature selection in logistic regression and support vector machines (SVMs). 
Their solution methods are based on the difference of convex functions algorithm (DCA),
where each iteration involves solving a linear program.
%In \cite{Cheng2013}, DCA is used to solve an approximate $\ell_0$ regularized SVMs with
%$\mu$-ramp loss, where the regularizer is a concave function,
%and capped $\ell_1$ regularized SVMs with hinge loss and with $\mu$-ramp loss.
%In the work \cite{Yang2016}, a mixed $\ell_2$-$\ell_0$ norm is proposed for
%sparse feature selection in logistic regression, and the $\ell_0$ pseudo norm
%is also approximated by a concave function, and the DC problem is devised with
%a solution method based on DCA.
In this work, our regularizer also belongs to the general class of DC functions,
but we study a more specific class, i.e., the weakly convex functions,
and there is no need to solve a linear program in every iteration to solve the problem,
given that the proximal operator of the weakly convex function has a closed form expression.

\subsection{Nonconvex compressed sensing}

From the perspective of reconstructing $\bm{\theta}$,
one-bit compressed sensing \cite{Boufounos2008} studies a similar problem,
where a sparse vector $\bm{\theta}$ (or its normalization $\bm{\theta}/\|\bm{\theta}\|_2$) 
is to be estimated from several one-bit measurements
$y^{(i)} = {\bf 1}(\bm{\theta}^{\rm T} {\bf x}^{(i)} \geq 0)$,
and compressed sensing \cite{Donoho2006} studies a problem
where a sparse $\bm{\theta}$ is to be estimated from several linear measurements
$y^{(i)} = \bm{\theta}^{\rm T} {\bf x}^{(i)}$.
In this setting ${\bf x}^{(i)}$ for $i = 1,\ldots,N$ are known sensing vectors.
Nonconvex regularizations have been used to promote sparsity
in both compressed sensing \cite{Chen2014,Chen2014c,Chen2015,Shen2016}
and one-bit compressed sensing \cite{zhu2015}.
These studies have shown that, despite that nonconvex optimization problems are usually hard to solve globally, with some proper choices of the nonconvex regularizers,
using some local methods 
their recovery performances can be better than that of the $\ell_1$ regularization,
both theoretically and numerically,
in terms of required number of measurements and robustness against noise.

\subsection{Weakly convex sparsity inducing function}

A class of weakly convex functions has been proposed to induce sparsity \cite{Chen2014}.
The definition is as the following.
\begin{definition}\label{def:1}\cite{Chen2014}
	The weakly convex sparsity inducing function $J$ is defined to be separable
	$$J({\bf x}) = \sum_{i=1}^n F({\bf x}_i),$$ 
	where the function $F:\reals \to \reals_+$ satisfies the following properties.
	\begin{itemize}
		\item Function $F$ is even and not identically zero, and $F(0) = 0$;
		\item Function $F$ is non-decreasing on $[0,\infty)$;
		\item The function $t\mapsto F(t)/t$ is nonincreasing on $(0,\infty)$;
		\item Function $F$ is weakly convex \cite{Vial1983} 
		on $[0,\infty)$ with nonconvexity parameter
		$\zeta>0$, i.e., $\zeta$ is the smallest positive scalar such that the function
		$$H(t) = F(t) + \zeta t^2$$ is convex.
	\end{itemize}
\end{definition}
According to the definition function $J$ is weakly convex,
and 
\[G({\bf x}) = J({\bf x}) + \zeta \|{\bf x}\|_2^2 = \sum_{i = 1}^d H({\bf x}_i)\]
 is a convex function.
Thus, $J$ belongs to a wider class of DC functions \cite{LeThi2008}, 
and both $J$ and $G$ are separable across all coordinates.
Since $\zeta>0$, the function $J$ is nonconvex,
and it can be nondifferentiable, which indicates that an optimization problem
with $J$ in the objective function can be hard to solve.
Nevertheless, the fact that by adding a quadratic term
the function becomes convex allows it to have some favorable properties,
such as that its proximal operator is well defined by a convex problem 
with a unique solution.

The proximal operator $\prox_{\beta J}(\cdot)$ of function $J$
with parameter $\beta$ is defined as
\begin{equation} \label{prox}
\prox_{\beta J} ({\bf v}) = 
\argmin \beta J({\bf x}) + \frac{1}{2} \|{\bf x}-{\bf v}\|_2^2,
\end{equation}
where the minimization is with respect to $\bf x$.
If $\beta$ is small enough so that $\beta \zeta <\frac{1}{2}$,
then the objective function in \eqref{prox} is strongly convex,
and the minimizer is unique.
For some weakly convex functions, 
their proximal operators have closed form expressions
which are relatively easy to compute.

For instance a specific $F$ satisfying Definition \ref{def:1} is defined as follows
\begin{align}\label{F}
F(t) = \left\{
\begin{array}{ll}
|t| - \zeta t^2 \quad &|t| \leq \frac{1}{2\zeta} \\
\frac{1}{4\zeta} \quad &|t| > \frac{1}{2\zeta}
\end{array}.
\right.
\end{align}
The function in \eqref{F} is also called minimax concave
penalty (MCP) proposed in \cite{Zhang2010} for penalized variable selection in linear regression, and has been used in both sparse logistic regression \cite{Loh2013}
and compressed sensing \cite{zhu2015,Shen2016}.
Its proximal operator with $\beta \zeta <\frac{1}{2}$ can be explicitly written as 
\begin{align}\label{prox_F}
\mathrm{prox}_{\beta F} (v) = \left\{
\begin{array}{ll}
0 & |v| < \beta \\
\frac{v - \beta \mathrm{sign}(v)}{1-2\beta\zeta} & \beta\leq |v| \leq \frac{1}{2\zeta}\\
v & |v| > \frac{1}{2\zeta}
\end{array}.
\right.
\end{align}
The proximal operator \eqref{prox_F} is also called firm shrinkage
operator \cite{Gao1997}, which generalizes the hard and soft shrinkage
corresponding to the proximal operators of the
$\ell_1$ norm and the pseudo $\ell_0$ norm, respectively.

\section{Sparse logistic regression with weakly convex regularization}
\label{sec:sparse_log_reg}

To fit a sparse logistic regression model,
we propose to try to solve problem \eqref{main}
with function $J$ belonging to the class of weakly convex
sparsity inducing functions in Definition \ref{def:1}.
Note that when the nonconvexity parameter $\zeta = 0$,
problem \eqref{main} becomes convex
and the standard $\ell_1$ logistic regression is an instance of it.

\subsection{Convexity}
An interesting observation is that at this point
problem \eqref{main} with the nonconvexity parameter $\zeta>0$
can either be convex
or nonconvex, depending on the data matrix 
$${\bf X} = \left({\bf x}^{(1)},\ldots, {\bf x}^{(N)}\right),$$
the regularization parameter $\beta$, as well as the nonconvexity
parameter $\zeta$.
In the following, from a perspective we have a conclusion that problem \eqref{main}
is nonconvex with any $\zeta>0$ (not necessarily sufficiently large).

\begin{theorem}\label{thm:convexity}
	If the dimension of the column space of the matrix $\bf X$
	is less than $d$, i.e.,
	matrix $\bf X$ does not have full row rank,
	then problem \eqref{main} is nonconvex for any $\zeta>0$.
\end{theorem}

\begin{remark}
	For the data matrix $\bf X$,
	if the number of data points $N$ is less than the dimension $d$,
	which is a typical situation where regularization is needed,
	or the data points are on a low dimensional subspace in $\reals^d$,
	then the dimension of the column space of $\bf X$ is less than $d$.
	In this work, we do not require that $\bf X$ has full row rank,
	so in general the problem \eqref{main} that we try to solve is nonconvex.
\end{remark}

\begin{proof}
If problem \eqref{main} is convex, i.e., its objective is convex,
then the subgradient set of the objective at any $\bm{\theta}$ is
\[
\nabla l(\bm{\theta}) + \beta \partial G(\bm{\theta}) - 2\beta\zeta\bm{\theta},
\]
where the plus and minus signs operate on every element of the
set $\partial G(\bm{\theta}) $. We denote 
$ \partial G(\bm{\theta}) - 2\zeta\bm{\theta}$ as $\partial J(\bm{\theta})$.
If  \eqref{main} is convex,
the following must hold for any $\bm{\theta}_0, \bm{\theta}$ and
any subgradient ${\bf g}\in\nabla l(\bm{\theta}_0) + \beta \partial J(\bm{\theta}_0)$
\begin{align}\label{cvx}
l(\bm{\theta}) + \beta J(\bm{\theta}) \geq l(\bm{\theta}_0) + \beta J(\bm{\theta}_0)
+ {\bf g}^{\rm T}(\bm{\theta} - \bm{\theta}_0).
\end{align}
In the following we will construct $\bm{\theta}_0$, $\bm{\theta}$, and $\bf g$
such that \eqref{cvx} does not hold.

From Definition \ref{def:1}, $H(t) = F(t) + \zeta t^2$ is convex,
so $H^{\prime}_-(t)$ and $H^{\prime}_+(t)$ always exist,
and $F^{\prime}_-(t)$ and $F^{\prime}_+(t)$ also exist according to
\[H^{\prime}_-(t) = F^\prime_-(t) + 2\zeta t \leq F^\prime_+(t) + 2\zeta t = H^{\prime}_+(t).\]
From Definition \ref{def:1} we also have  that  for all $t_1>t_2>0$
$$ \frac{F(t_1)}{t_1} \leq \frac{F(t_2)}{t_2} \leq F^\prime_+(0).$$
Because $F$ is not linear, there must exist $t_0>0$ 
such that for all $t>t_0$
\begin{align}\label{large_t}
\frac{F(t)}{t} \leq \frac{F(t_0)}{t_0} < F^\prime_+(0).
\end{align}
Note that $F^\prime_+(0)  = -F^\prime_-(0)>0$ is true,
because $F$ is even, and if $F^\prime_+(0) = 0$, then $F(t)/t \leq 0$
and $F(t) \geq 0$ will lead to $F(t) = 0$ for all $t>0$, which does not satisfy 
Definition \ref{def:1}.
	
%The Hessian of $f$ at a point $\theta$ is known to be 
%$$\nabla^2 f(\theta) = X\Sigma(\theta) X^{\rm T},$$
%where $\Sigma(\theta)$ is a diagonal matrix with 
%$$\Sigma(\theta)_{ii} = \sigma \left( \theta^{\rm T} x^{(i)}\right) 
%\left(1- \sigma \left(\theta^{\rm T} x^{(i)}\right) \right).$$
%We know that $\Sigma(\theta)_{ii} \in (0,1/4]$,
%and $\Sigma(\theta)_{ii} = 1/4$ when and only when $\theta^{\rm T} x^{(i)} = 0$.
%Because $X$ does not have full row rank, there exists $u\neq 0$ such that
%$u^{\rm T} \nabla^2 f(\theta) u = 0$ for any $\theta$,
%so the Hessian of $f$ at any point is only positive semi-definite 
%(not positive definite).
Because $\bf X$ does not have full row rank, there exists ${\bf u} \neq {\bf 0}$ such that
${\bf u}^{\rm T} {\bf X} = 0$. For such $\bf u$, we have that
\[
l(t {\bf u}) = l({\bf 0}) 
\]
holds for any $t$.

Next we will find $t>0$ and
${\bf g} \in\nabla l({\bf 0}) + \beta \partial J({\bf 0})$,
such that
\[
\beta J(t{\bf u}) < \beta J({\bf 0}) + (t{\bf u} - {\bf 0})^{\rm T} ({\bf g} - \nabla l({\bf 0})).\]
Note that $J({\bf 0}) = 0$.
Suppose that for any $t>0$ and any
${\bf h} = ({\bf g} - \nabla l({\bf 0}))/\beta \in \partial J({\bf 0}) $,
which is equivalent to ${\bf h}_i\in[F^\prime_-(0), F^\prime_+(0)]$,
the following holds
\begin{align}\label{cvx2}
\sum_{i=1}^d F(t{\bf u}_i) =
J(t{\bf u}) \geq t {\bf u}^{\rm T} {\bf h} = t \sum_{i=1}^d {\bf h}_i {\bf u}_i.
\end{align}
Because of \eqref{large_t}, for every ${\bf u}_i>0$, there is a $t_i>0$ such that for all
$t>t_i$ 
\[
F(t {\bf u}_i) < F^{\prime}_+ (0) t {\bf u}_i,
\]
and for every ${\bf u}_i<0$ there is a $t_i>0$ such that for all
$t>t_i$ 
\[
F(t{\bf u}_i) < t F^\prime_-(0) {\bf u}_i.
\]
Thus, we have that the following holds for all $t>\max_i(t_i)$
\[
\sum_{i=1}^d F(t{\bf u}_i) 
< \sum_{{\bf u}_i > 0} t F^\prime_+(0) {\bf u}_i
+ \sum_{{\bf u}_i<0} tF^\prime_-(0) {\bf u}_i.
\]
In \eqref{cvx2}, by taking ${\bf h}_i = F^\prime_+(0)$ when ${\bf u}_i >0$
and ${\bf h}_i = F^\prime_-(0)$ when ${\bf u}_i<0$,
we have a contradiction.
Now we have proved that
\begin{align*}
l(t{\bf u}) + \beta J(t{\bf u})
&< l({\bf 0})  +{\bf g}^{\rm T} t{\bf u}\end{align*}
holds for some $t>0$
and a ${\bf g}\in\nabla l({\bf 0})+\beta \partial J({\bf 0})$,
so \eqref{cvx} does not hold for all $\bm{\theta}_0$ and $\bm{\theta}$,
and the objective in \eqref{main} is nonconvex.

%According to Definition \ref{def:1}
%the function $H(t) = F(t) + \zeta t^2$ is convex,
%so the left and right derivatives of $H$ always exist,
%and its left derivative is no larger than its right derivative
%at any point, i.e., 
%\[F^\prime_-(t) + 2\zeta t \leq F^\prime_+(t) + 2\zeta t.\]
%Thus, $F^\prime_-(t)  \leq F^\prime_+(t)$, and the inequality is strict if
%$F$ is not differentiable at $t$.
%Since there are only finite points where $F$ is not differentiable,
%$F^\prime_+(t) > F^\prime_-(t)$ holds at these finite points.
%From Definition \ref{def:1} we have 
%\[
%\frac{F(a) - F(b)}{a - b}\leq \frac{F(b)}{b}, \quad \forall a>b>0,
%\]
%so we know that  for all $b>0$
%$$F^\prime _+(b) \leq \frac{F(b)}{b} \leq F^\prime_+(0). $$
%Since $F$ is not linear,
%there exists $b_0>0$ such that $$F^\prime_+(b_0) < F^\prime_+(0).$$
%Therefore, there exists an interval $[b_1, b_2]\subset[0,b_0]$ such that
%$F^\prime_+(b)$ is smooth on it with $F^\prime_+(b_1)>F^\prime_-(b_2)$.
%Otherwise, together with the fact that $F^\prime_-(t) \leq F^\prime_+(t)$,
%there will be a contradiction that $F^\prime_+(0) \leq F^\prime_+(b_0).$
%Thus, there exists $b_3\in [b_1, b_2]$ such that $F^{\prime\prime}(b_3)<0$.
%At the point $\theta = (b_3, \ldots,b_3)$, the Hessian of the objective function
%is the sum of a positive semi-definite matrix (not positive definite)
%and a diagonal matrix with diagonal elements all negative,
%so it has negative eigenvalue, and the objective function is not convex.
\end{proof}

\subsection{Local optimality conditions}
In this part we discuss optimality conditions for problem \eqref{main}.
As revealed in the previous part, problem \eqref{main} can easily be nonconvex,
so its local optimality conditions are worth studying.
First we will have a sufficient condition for local optimality,
and next a necessary condition is unveiled.

What has already been known is that, for a DC function, 
its local minimum has to be a critical point \cite{An2005}
which is defined as the following.
	\begin{definition}\cite{An2005}
		A point ${\bf x}^\ast$ is said to be a critical point of 
		a DC function $g({\bf x}) - h({\bf x})$, where $g({\bf x})$ and $h({\bf x})$ are convex,
		if $\partial g({\bf x}^\ast) \cap \partial h({\bf x}^\ast) = \emptyset$.
	\end{definition}
When a function is differentiable, the above definition
is in consistent with the common definition that a critical point is a point
where the derivative equals zero.
Consequently, in our settings if $\bm{\theta}^\ast$ is a local optimum of problem
\eqref{main}, then we at least know that
$2\beta\zeta \bm{\theta}^\ast \in \partial (l+\beta G) (\bm{\theta}^\ast)$,
which is equivalent to 
\begin{equation}\label{criticalP}
	2\zeta \bm{\theta}^\ast - \frac{1}{\beta}\nabla l(\bm{\theta}^\ast) \in \partial G(\bm{\theta}^\ast).
\end{equation}
In the following we have some further conclusions,	
and let us begin with a theorem on a sufficient local optimality condition. 
\begin{theorem}\label{thm:local_opt_suf}
	Suppose that $H^{\prime\prime}_-(t)$ and $H^{\prime\prime}_+(t)$ exist for any $t\in\reals$ at which $H(t)$ is differentiable.
	If for every $\bm{\theta}^\ast_i$, $i = 1,\ldots,d$, one of the following conditions holds,
	then $\bm{\theta}^\ast$ is a local optimum of problem \eqref{main}.
	\begin{itemize}
		\item Function $F$ is not differentiable at $\bm{\theta}^\ast_i$,
		and 
		\begin{equation}\label{cond1}
		2\zeta\bm{\theta}^\ast_i - \frac{1}{\beta}\nabla l(\bm{\theta}^\ast)_i \in (H^\prime_-(\bm{\theta}^\ast_i), H^\prime_+(\bm{\theta}^\ast_i)).
		\end{equation}
		\item Function $F$ is differentiable at $\bm{\theta}^\ast_i$, 
		\begin{equation}\label{cond2}
		2\zeta\bm{\theta}^\ast_i - \frac{1}{\beta}\nabla l(\bm{\theta}^\ast)_i = H^\prime(\bm{\theta}^\ast_i),
		\end{equation}
		and
		both $H^{\prime\prime}_+(\bm{\theta}^\ast_i) $ and $H^{\prime\prime}_-(\bm{\theta}^\ast_i)$ are no less than $2\zeta$.
	\end{itemize}
	
\end{theorem}

\begin{remark}
Function $H$ is convex,
and $F$ being not differentiable at $\bm{\theta}^\ast_i$
is equivalent to that $H$ is not differentiable at $\bm{\theta}^\ast_i$,
so $H^\prime_-(\bm{\theta}^\ast_i)< H^\prime_+(\bm{\theta}^\ast_i)$,
and the open interval in \eqref{cond1} exists.
\end{remark}

\begin{remark}
	The conditions \eqref{cond1} and \eqref{cond2} imply \eqref{criticalP},
	which is a necessary condition for $\bm{\theta}^\ast$ to be a local optimum
	and boils down to 
	\begin{align}\label{inclusion}
		H^\prime_-(\bm{\theta}^\ast_i) \leq 2\zeta\bm{\theta}^\ast_i - \frac{1}{\beta}\nabla l(\bm{\theta}^\ast)_i  \leq  H^\prime_+(\bm{\theta}^\ast_i)
	\end{align}
	for every $i$.
	As a sufficient condition, Theorem \ref{thm:local_opt_suf} requires more
	than \eqref{criticalP}. A direct observation is that \eqref{cond1} requires strict inequalities, while the ones in \eqref{inclusion} are not strict.
\end{remark}
	
\begin{proof}
	By definition of local optimality, 
	$\bm{\theta}^\ast$ is a local optimal point of problem \eqref{main},
	if and only if
	\begin{align*}
	\beta\zeta \|\bm{\theta}\|_2^2 +
	l(\bm{\theta}^\ast) + \beta G(\bm{\theta}^\ast) - \beta\zeta \|\bm{\theta}^\ast\|_2^2
	\leq l(\bm{\theta}) + \beta G(\bm{\theta}) 
	\end{align*}
	holds for any $\bm{\theta}$ in a small neighborhood of $\bm{\theta}^\ast$.
	It can be equivalently written as
	\begin{align}\label{local_opt_ineq}
	l(\bm{\theta}^\ast) + \beta G(\bm{\theta}^\ast) 
	\leq l(\bm{\theta}) + \beta G(\bm{\theta}) 
	+2\beta\zeta \langle \bm{\theta}^\ast-\bm{\theta}, \bm{\theta}^\ast \rangle 
	-\beta\zeta \|\bm{\theta}^\ast - \bm{\theta}\|_2^2.
	\end{align}
	We will prove that if every $\bm{\theta}^\ast_i$ for $i = 1,\ldots, d$
	satisfies either of the two conditions, then $\bm{\theta}^\ast$ is a local
	optimum, i.e., \eqref{local_opt_ineq} holds in a small neighborhood
	of $\bm{\theta}^\ast$.
	
	If the first condition \eqref{cond1} holds for $\bm{\theta}^\ast_i$, then together with the convexity inequalities that
	\begin{align*}
	H(\bm{\theta}_i^\ast) &\leq H(\bm{\theta}_i) + (\bm{\theta}_i^\ast - \bm{\theta}_i) H^\prime_- (\bm{\theta}_i^\ast) \\
	H(\bm{\theta}_i^\ast) &\leq H(\bm{\theta}_i) + (\bm{\theta}_i^\ast - \bm{\theta}_i) H^\prime_+ (\bm{\theta}_i^\ast),
	\end{align*}
	we know that 
	\[
	H(\bm{\theta}_i^\ast) \leq H(\bm{\theta}_i) + (\bm{\theta}_i^\ast - \bm{\theta}_i) 
	\left(2\zeta\bm{\theta}^\ast_i - \frac{\nabla l(\bm{\theta}^\ast)_i}{\beta} - \zeta(\bm{\theta}^\ast_i - \bm{\theta}_i)\right)
	\]
	holds for all $\bm{\theta}_i$ such that
	$$0\leq \zeta(\bm{\theta}_i^\ast-\bm{\theta}_i)
	\leq 2\zeta\bm{\theta}^\ast_i - \frac{\nabla l(\bm{\theta}^\ast)_i}{\beta} -  H^\prime_-(\bm{\theta}^\ast_i), $$
	and all $\bm{\theta}_i$ such that
	$$0\leq \zeta(\bm{\theta}_i-\bm{\theta}_i^\ast)\leq H^\prime_+(\bm{\theta}^\ast_i) - 2\zeta\bm{\theta}^\ast_i + \frac{\nabla l(\bm{\theta}^\ast)_i}{\beta}.$$
	Therefore, there exists $\delta_i >0$ such that 
	for all $(\bm{\theta}_i - \bm{\theta}_i^\ast)^2 \leq\delta_i$ we have
	\begin{align}\label{coor_local_opt_ineq}
		H(\bm{\theta}_i^\ast) \leq  H(\bm{\theta}_i ) + (\bm{\theta}_i^\ast - \bm{\theta}_i ) (2\zeta\bm{\theta}^\ast_i - \nabla l(\bm{\theta}^\ast)_i /\beta)
		- \zeta (\bm{\theta}_i  -  \bm{\theta}_i^\ast)^2. 
	\end{align}
		
	If the second condition in Theorem \ref{thm:local_opt_suf} holds for $\bm{\theta}^\ast_i$,
	then according to the second order Taylor expansion, there exists $\delta_i>0$ such that the following holds in a small neighborhood 
	$(\bm{\theta}_i  - \bm{\theta}_i ^\ast)^2 \leq \delta_i$,
	\[
	H(\bm{\theta}_i^\ast) + (\bm{\theta}_i  - \bm{\theta}_i^\ast) H^{\prime}(\bm{\theta}_i ^\ast) + \zeta (\bm{\theta}_i  -  \bm{\theta}_i^\ast)^2 \leq  H(\bm{\theta}_i ).
	\]
	Together with \eqref{cond2}, we also arrive at \eqref{coor_local_opt_ineq}.
	Therefore,  \eqref{coor_local_opt_ineq} holds for every coordinate with some $\delta_i>0$, and we have that
	\[
	G(\bm{\theta}^\ast) \leq G(\bm{\theta}) + (2\zeta \bm{\theta}^\ast - \nabla l(\bm{\theta}^\ast)/\beta)^{\rm T} (\bm{\theta} - \bm{\theta}^\ast) - \zeta \|\bm{\theta} - \bm{\theta}^\ast\|^2
	\]
	holds for $\bm{\theta}$ in a neighborhood $\|\bm{\theta} - \bm{\theta}^\ast \|^2 \leq \delta$ with $\delta = \min_i \delta_i >0$.
	Together with the fact that $l$ is convex, we have that
     \eqref{local_opt_ineq} holds in such neighborhood.
\end{proof}

Next, we will show a necessary condition for $\theta^\ast$ to be a local optimum.
\begin{theorem}\label{thm:local_opt_nec}
	Suppose that $H^{\prime\prime}_-(t)$ and $H^{\prime\prime}_+(t)$ exist
	for any $t\in\reals$ at which $H(t)$ is differentiable.
	If $\bm{\theta}^\ast$ is a local optimum of problem \eqref{main},
	then for every $\bm{\theta}^\ast_i$, $i = 1,\ldots,d$, one of the following
	conditions holds.
	\begin{itemize}
		\item Function $F$ is not differentiable at $\bm{\theta}^\ast_i$ and \eqref{inclusion} holds.
		\item Function $F$ is differentiable at $\bm{\theta}^\ast_i$,
		\eqref{cond2} holds,
		and both $H^{\prime\prime}_+(\bm{\theta}^\ast_i) $ and $H^{\prime\prime}_-(\bm{\theta}^\ast_i)$ are no less than 
		$2\zeta - 0.25 \|{\bf X}\|^2/\beta$.
	\end{itemize}
\end{theorem}

\begin{remark}
	Theorem \ref{thm:local_opt_suf} and Theorem \ref{thm:local_opt_nec} 
	can be used to certify if a point is a local optimum.
	We can see that there is a gap between them.
	If function $F$ is not differentiable at $\bm{\theta}^\ast_i$,
	the sufficient condition requires \eqref{inclusion} to hold with strict
	inequalities, while the necessary condition does not require them
	to be strict.
	If function $F$ is differentiable at $\bm{\theta}^\ast_i$,
	then the sufficient condition requires the left and right second
	derivatives of $H$ to be no less than $2\zeta$,
	while the necessary condition requires them to be no less than
	$2\zeta - 0.25 \|{\bf X}\|^2/\beta$, which is due to the contribution of function $l$
	to the convexity.
\end{remark}

\begin{proof}
	Since $\bm{\theta}^\ast$ is a local optimum,
	it is a critical point, so we only need to prove that for a critical point $\bm{\theta}^\ast$,
	if there exists $\bm{\theta}^\ast_i$ that does not
	satisfy the two conditions, then such a critical point
	$\bm{\theta}^\ast$ cannot be a local optimum.
	Suppose that there is a 
	$\bm{\theta}^\ast_i$ at which $F$ (and also $H$) is differentiable, and one of 
	$H^{\prime\prime}_-(\bm{\theta}^\ast_i)$ and $H^{\prime\prime}_+(\bm{\theta}^\ast_i)$
	is less than $2\zeta - 0.25 \|{\bf X}\|^2/\beta$. 
	Without loss of generality, we assume that
	\begin{align}\label{H_hessian}
	H^{\prime\prime}_+(\bm{\theta}^\ast_i) < 2\zeta - 0.25 \|{\bf X}\|^2/\beta.
	\end{align}
	Then we take 
	\[\bm{\theta} = (\bm{\theta}^\ast_1, \ldots, \bm{\theta}^\ast_{i-1}, \bm{\theta}^\ast_i - t,
	\bm{\theta}^\ast_{i+1},\ldots,\bm{\theta}^\ast_d)\]
	 for $t > 0$, so
	\[
	\bm{\theta} - \bm{\theta}^\ast = (0,\ldots, 0, -t, 0,\ldots,0),
	\]
	and
	\begin{align*}
	 l(\bm{\theta}) + \beta G(\bm{\theta}) - l(\bm{\theta}^\ast) - \beta G(\bm{\theta}^\ast)  
	& = l(\bm{\theta})  -  l(\bm{\theta}^\ast)  + \beta H(\bm{\theta}^\ast_i- t) - \beta H(\bm{\theta}^\ast_i ) \\
	& \leq   -\nabla l(\bm{\theta}^\ast)_i t + \frac{1}{8}  \|{\bf X}\|^2  t^2
	+  \beta H(\bm{\theta}^\ast_i- t) - \beta H(\bm{\theta}^\ast_i )\\
	& <   -\nabla l(\bm{\theta}^\ast)_i t + \frac{1}{8} \|{\bf X}\|^2  t^2
	- \beta  H^{\prime}(\bm{\theta}^\ast_i ) t 
	 +\left(\beta  \zeta -\frac{1}{8} \|{\bf X}\|^2 \right) t^2 \\
	&= - \nabla l(\bm{\theta}^\ast)_i t  
	-\beta  H^{\prime}(\bm{\theta}^\ast_i ) t
	+\beta \zeta t^2,
	\end{align*}
	where the first inequality holds in that the Lipchitz constant of $\nabla l$ is $0.25\|X\|^2$, and
	the second inequality holds for small positive $t$ according to the second order Taylor expansion and \eqref{H_hessian}.
	Consequently, together with \eqref{cond2} we have that in any small neighborhood there exists
	$\bm{\theta}$ such that 
		\begin{align*}
		l(\bm{\theta}^\ast) + \beta G(\bm{\theta}^\ast) > l(\bm{\theta}) + \beta G(\bm{\theta}) + 2\beta\zeta (\bm{\theta}^\ast - \bm{\theta})^{\rm T} \bm{\theta}^\ast 
		- \beta \zeta \|\bm{\theta}-\bm{\theta}^\ast\|^2
		\end{align*}
	holds, so $\bm{\theta}^\ast$ cannot be a local optimum.
	
\end{proof}

\subsection{Choice of the regularization parameter}
In this part, we will show a condition on the choice of the
regularization parameter $\beta$ to avoid ${\bf 0}$ to become
a local optimum of problem \eqref{main}.
More specifically, in the following,
we will show that if $\beta$ is larger than a certain value,
then ${\bf 0}$ will be a local optimum of problem \eqref{main},
and if $\beta$ is smaller than that value,
then $\bf 0$ will not even be a critical point.

\begin{theorem}\label{thm:beta}
	Suppose that the mean of the data points is subtracted from them, 
	i.e., $\sum_{i = 0}^N {\bf x}^{(i)} = 0$.
	If
	\[
	\beta < \frac{ \left\| \sum_{y^{(i)} = 1} {\bf x}^{(i)} \right\|_\infty }{F^\prime_+(0)},
	\]
	then $\bf 0$ is not a local minimum of problem \eqref{main}.
	If 
	\[
	\beta > \frac{ \left\| \sum_{y^{(i)} = 1} {\bf x}^{(i)} \right\|_\infty }{F^\prime_+(0)},
	\]
	then $\bf 0$ is a local minimum of problem \eqref{main},
\end{theorem}

\begin{proof}
	If $-\nabla l({\bf 0})  \notin \beta\partial G({\bf 0})$,
	then according to \eqref{criticalP}, $\bf 0$ is not a critical point of the objective
	in problem \eqref{main},
	so $\bf 0$ is not a local optimum of problem \eqref{main}.
	Since $\sum_{i = 0}^N {\bf x}^{(i)} = 0$, the condition becomes
	\[
		-\nabla l({\bf 0})  = \sum_{y^{(i)} = 1} {\bf x}^{(i)} \notin \beta \partial G({\bf 0}), 
	\]
	which is equivalent to that
	\[	    
	\beta F^\prime_+(0) <  \bigg\| \sum_{y^{(i)} = 1} {\bf x}^{(i)} \bigg\|_\infty.
	\]
	
	On the other side, if
	\[	    
		\beta F^\prime_+(0) > \bigg\| \sum_{y^{(i)} = 1} {\bf x}^{(i)} \bigg\|_\infty,
	\]
	then $\bf 0$ is a local optimum of problem \eqref{main}.
	To prove this, first notice that the following holds
	for any $\bm{\theta}$ and any ${\bf g}\in\partial G({\bf 0})$
	\begin{align*}
	l(\bm{\theta}) &\geq l({\bf 0}) + \bm{\theta}^{\rm T} \nabla l({\bf 0}), \\
	J(\bm{\theta}) &\geq G({\bf 0}) + {\bf g}^{\rm T} \bm{\theta} - \zeta \|\bm{\theta}\|_2^2.
	\end{align*}
	Thus, for any $\bm{\theta}$ and any ${\bf g}\in\partial G({\bf 0})$
	\[
	l(\bm{\theta}) + \beta J(\bm{\theta}) \geq l({\bf 0}) + \beta J({\bf 0}) + \bm{\theta}^{\rm T} \nabla l({\bf 0}) + \beta {\bf g}^{\rm T} \bm{\theta} - \beta\zeta\|\bm{\theta}\|_2^2.
	\]
	Then, for every $\bm{\theta}$ we take the following ${\bf g}\in\partial G({\bf 0})$
	\begin{align*}
	{\bf g}_i = \left\{
	\begin{array}{ll}
	F^\prime_+(0), &\bm{\theta}_i \geq 0 \\
	F^\prime_-(0),  &\bm{\theta}_i < 0,
	\end{array}
	\right.
	\end{align*}
	so we have
	\begin{align*}
	\bm{\theta}^{\rm T} (\nabla l({\bf 0}) +\beta {\bf g}) - \beta\zeta\|\bm{\theta}\|_2^2 
	= & \sum_{\bm{\theta}_i>0} \bm{\theta}_i (\nabla l({\bf 0})_i + \beta F^\prime_+(0)) 
	- \beta\zeta \bm{\theta}_i^2
	+ \sum_{\bm{\theta}_i<0} \bm{\theta}_i(\nabla l({\bf 0})_i +\beta F^\prime_-(0)) 
	- \beta\zeta \bm{\theta}_i^2 \\
	 > & 0,
	\end{align*}
	where the last inequality holds for all $\|\bm{\theta}\|_2$ small enough,
	in that $\nabla l({\bf 0})_i +\beta F^\prime_+(0)$ is strictly positive
	and $\nabla l({\bf 0})_i +\beta F^\prime_-(0)$ is strictly negative.
	Therefore, we have that in a small neighborhood of $\bf 0$,
	the following holds
	\[
	l(\bm{\theta}) + \beta J(\bm{\theta}) > l({\bf 0}) + \beta J({\bf 0}),
	\]
	which means that $\bf 0$ is a local minimum in this case.
	
	In the proof of Theorem \ref{thm:convexity}, we have that
	$F^\prime_+(0)>0$, so we reach the conclusions in Theorem \ref{thm:beta}.

\end{proof}

\section{A proximal gradient method}
\label{sec:solution_method}

In this section, we try to solve the weakly convex regularized sparse logistic regression problem \eqref{main}
 with any function $J$ satisfying Definition \ref{def:1}.
Since the logistic loss $l$ is differentiable and the proximal operator of
function $J$ can be well defined,
the method that we use is proximal gradient descent,
and the iterative update is as the following
\begin{equation}\label{alg}
\bm{\theta}_{k+1} = \prox_{\alpha_k\beta J} (\bm{\theta}_k - \alpha_k \nabla l(\bm{\theta}_k)),
\end{equation}
where $\alpha_k>0$ is a stepsize, 
and 
\begin{align*}
\nabla l(\theta_k) 
&= \sum_{i=1}^N \left( \sigma \left( \bm{\theta}_k^{\rm T}{\bf x}^{(i)} \right) - y^{(i)} \right) {\bf x}^{(i)}.
\end{align*}
Note that the update \eqref{alg} of the algorithm is equivalent to solving the following problem
\[
\mbox{minimize} \quad \alpha_k \beta J(\bm{\theta}) 
+ \frac{1}{2} \left\| \bm{\theta} - \bm{\theta}_k +  \alpha_k \nabla l(\bm{\theta}_k) \right\|_2^2,
\]
which is strongly convex for $\alpha_k \beta\zeta <1/2$.
The computation of the gradient can be distributed
in every $i$
and then summed up,
and the calculation of the proximal operator can be elementwise parallel,
in that the function $J$ is separable across the coordinates
according to its definition.

The stepsize $\alpha_k$ in the algorithm
can be chosen as a constant $\alpha$
or determined by backtracking.
In the following we prove its convergence
with the two stepsize rules.
\begin{theorem}\label{thm:converge}
	For stepsize $\alpha_k$ chosen from one of the following ways,
	\begin{itemize}
		\item  constant stepsize $\alpha_k = \alpha$
		and 
		\begin{equation}\label{alpha}
		\frac{1}{\alpha} > \max\left (2\beta\zeta, \frac{1}{8}\|{\bf X}\|^2 +\beta\zeta\right);
		\end{equation}
		\item  backtracking stepsize $\alpha_k = \eta^{n_k} \alpha_{k-1}$,
		where $\beta\zeta\alpha_0 < 1/2$, $0<\eta<1$, and $n_k$ is the smallest nonnegative integer for the following to hold
		\begin{align*}
		l(\bm{\theta}_k) \leq 
		&l(\bm{\theta}_{k-1}) + \langle \bm{\theta}_k - \bm{\theta}_{k-1},
		\nabla l(\bm{\theta}_{k-1}) \rangle
		+ \frac{1}{2\alpha_k} \| \bm{\theta}_{k-1} - \bm{\theta}_{k} \|_2^2
		\end{align*}
	\end{itemize}
	the sequence $\{\bm{\theta}_k\}$ 
	generated by the algorithm satisfies the following.
	\begin{itemize}
		\item  Objective function
	$l(\bm{\theta}_k) + \beta J(\bm{\theta}_k)$ is monotonically 
	non-increasing and convergent;
	\item The update of the iterates converges to $0$, i.e.,
	$$\|\bm{\theta}_{k} - \bm{\theta}_{k-1}\|_2 \rightarrow 0;$$
	\item
	The first order necessary local optimality condition will be approached,
	i.e., there exists ${\bf g}_{k} \in \partial G(\bm{\theta}_{k})$ for every $k$ such that
	\begin{equation}\label{claim3}
	\beta {\bf g}_{k}  - 2\beta\zeta \bm{\theta}_{k} + \nabla l(\bm{\theta}_{k})
	\rightarrow 0.
	\end{equation}
\end{itemize}
\end{theorem}

\begin{remark}
A constant stepsize depending on the maximum eigenvalue of
the data matrix $\bf X$ is able to guarantee convergence,
but when $\bf X$ is of huge size or distributed and its eigenvalue
is not attainable, a backtracking stepsize which does not depend
on such information can be used.
Note that because $l$ is Lipchitz differentiable,
$n_k$ in the backtracking method always exists.
\end{remark}

\begin{remark}
If the sequence $\{\bm{\theta}_k\}$ has limit points, then
the third conclusion means that every limit point of the sequence $\{\bm{\theta}_k\}$
is a critical point of the objective function.
\end{remark}

\begin{remark}
According to Theorem \ref{thm:converge},
the objective function converges,
so we can choose $\epsilon_{\mathrm{tol}}>0$ and set the following
\begin{align}\label{stopping}
|l(\bm{\theta}_{k+1}) + \beta J(\bm{\theta}_{k+1})
- l(\bm{\theta}_k) - \beta J(\bm{\theta}_k)| \leq \epsilon_{\mathrm{tol}}
\end{align}
 as a stopping criterion.
\end{remark}

\begin{proof}
	The techniques used in this proof are similar to the ones
	in \cite{Teboulle2009,beck2009}.
	To begin with, we are going to prove that the objective function in problem \eqref{main}
	is able to decrease monotonically during the iterations.
	First of all, notice the fact that the gradient of function $l$ is Lipchitz continuous,
	which gives the following inequality according to the Lipchitz property
	\begin{align}\label{L-property}
	l(\bm{\theta}_k) +\beta J(\bm{\theta}_k) 
	\leq  l(\bm{\theta}_{k-1}) + \nabla l(\bm{\theta}_{k-1})^{\rm T} (\bm{\theta}_k - \bm{\theta}_{k-1})
	 + \frac{L}{2} \|\bm{\theta}_k - \bm{\theta}_{k-1}\|_2^2 +\beta  J(\bm{\theta}_k),
	\end{align}
	where $L$ is the Lipchitz constant.
	If the backtracking stepsize is used, then we have
	\begin{align}\label{backtracking}
	l(\bm{\theta}_k) +\beta J(\bm{\theta}_k) \leq  
	l(\bm{\theta}_{k-1}) + \nabla l(\bm{\theta}_{k-1})^{\rm T}(\bm{\theta}_k - \bm{\theta}_{k-1})
	+ \frac{1}{2\alpha_k}\| \bm{\theta}_{k-1} - \bm{\theta}_{k} \|_2^2 +\beta  J(\bm{\theta}_k).
	\end{align}
	Secondly, according to our update rule,
	$\bm{\theta}_{k}$ minimizes the following function of $\bf u$
	\[
	\beta J({\bf u}) + \frac{1}{2\alpha_k} \|{\bf u}-\bm{\theta}_{k-1}\|_2^2 + \nabla l^{\rm T}(\bm{\theta}_{k-1}) ({\bf u}-\bm{\theta}_{k-1}),
	\]
	which is $\mu_k$-strongly convex given that
	$$\frac{\mu_k}{2} = \frac{1}{2\alpha_k} - \beta\zeta >0.$$
	Thus, we have that
	\begin{align}\label{mu_strongly}
	\beta J(\bm{\theta}_k) + \nabla l^{\rm T}(\bm{\theta}_{k-1}) (\bm{\theta}_k - \bm{\theta}_{k-1})
	+ \frac{1}{2\alpha_k} \| \bm{\theta}_{k} - \bm{\theta}_{k-1}\|_2^2
	\leq \beta J(\bm{\theta}_{k-1})  -  \frac{\mu_k}{2} \|\bm{\theta}_k - \bm{\theta}_{k-1}\|^2.
	\end{align}
	Combining \eqref{mu_strongly} with \eqref{backtracking} yields that
	\begin{align*}
	l(\bm{\theta}_k) + \beta J(\bm{\theta}_k) 
	\leq l(\bm{\theta}_{k-1})  + \beta J(\bm{\theta}_{k-1}) - \frac{\mu_k}{2} \|\bm{\theta}_k - \bm{\theta}_{k-1}\|_2,
	\end{align*}
	which means that the objective function is monotonically non-increasing
	with the backtracking stepsize.
	For the constant stepsize $\alpha = \alpha_k$, combining \eqref{mu_strongly} with \eqref{L-property}, we have
	\begin{align*}
	& \beta J(\bm{\theta}_k) + l(\bm{\theta}_k ) 
	\leq  l(\bm{\theta}_{k-1}) 
		+ \frac{L-1/\alpha - \mu}{2} \|\bm{\theta}_k - \bm{\theta}_{k-1}\|_2^2
		+\beta  J(\bm{\theta}_{k-1}).
	\end{align*}
	For $\alpha$ small enough such that
	$$L-1/\alpha-\mu < 0,$$
	we have that the objective function is non-increasing.
	Because
	\begin{align*}
	L \leq \sup_{\bm{\theta}} \|H(\bm{\theta})\|  = \frac{1}{4} \|{\bf X}\|^2,
	\end{align*}
	a sufficient condition for
	$$L-\frac{1}{\alpha} - \mu = L - \frac{2}{\alpha}  + 2\beta\zeta <0$$
	is that 
	$$\frac{1}{8}\|{\bf X}\|^2 +\beta\zeta < \frac{1}{\alpha},$$
	which is the requirement of the constant stepsize.
	Together with the fact that the objective function is lower bounded,
	we have proved the first claim in the Theorem \ref{thm:converge}.
	
	The second claim can be seen from that
	\begin{align*}
	0  \leq & \frac{1/\alpha+\mu-L}{2} \|\bm{\theta}_k - \bm{\theta}_{k-1}\|_2^2\\
	\leq & l(\bm{\theta}_{k-1}) + \beta J(\bm{\theta}_{k-1}) - (l(\bm{\theta}_{k}) + \beta J(\bm{\theta}_{k}))
	\end{align*}
	holds for the constant stepsize, and
		\begin{align*}
		0  \leq  &\frac{\mu_k}{2} \|\bm{\theta}_k-\bm{\theta}_{k-1}\|_2^2 \\
		\leq  & l(\bm{\theta}_{k-1}) + \beta J(\bm{\theta}_{k-1}) - (l(\bm{\theta}_{k}) + \beta J(\bm{\theta}_{k}))
		\end{align*}
	holds for the backtracking stepsize.
	Note that $\mu_k$ is nondecreasing during the iterations,
	so $\|\bm{\theta}_k - \bm{\theta}_{k-1} \|_2$ converges to $0$ in both cases.
	
	To see the third claim, remind that the update rule indicates that
	\begin{align*}
	{\bf 0} \in \beta \partial G(\bm{\theta}_{k})  - 2\beta\zeta \bm{\theta}_{k} + \nabla l(\bm{\theta}_{k-1})
	+\frac{1}{\alpha_k} (\bm{\theta}_{k}-\bm{\theta}_{k-1}),
	\end{align*}
	so there exists ${\bf g}_{k} \in G(\bm{\theta}_{k})$ for every $k$ such that
	\begin{align*}
	\beta {\bf g}_k - 2\beta\zeta\bm{\theta}_k + \nabla l(\bm{\theta}_k) 
	 =& \beta {\bf g}_k - 2\beta\zeta\bm{\theta}_k + \nabla l(\bm{\theta}_{k-1})
	 +  \nabla l(\bm{\theta}_k) - \nabla l(\bm{\theta}_{k-1}) \\
	 =&  - \frac{1}{\alpha_k} (\bm{\theta}_{k}-\bm{\theta}_{k-1}) 
	 +  \nabla l(\bm{\theta}_k) - \nabla l(\bm{\theta}_{k-1}).
	\end{align*}
	Since $\nabla l$ is continuous, 
	as $\bm{\theta}_{k} - \bm{\theta}_{k-1} \rightarrow {\bf 0}$,
	for a constant stepsize \eqref{claim3} holds.
	For the backtracking stepsize, $1/\alpha_k$ is nondecreasing,
	so \eqref{claim3} also holds.
\end{proof}

According to Theorem \ref{thm:converge}, the algorithm can be summarized
in Table \ref{alg:prox_g}.

\begin{table}
	\caption{
		Proximal gradient descent for weakly convex regularized logistic regression.
	}
	\centering
	\begin{tabular}{l}
		\toprule \textbf{Input}: initial point $\bm{\theta}_0$, $\alpha_0<1/(2\beta\zeta)$ (or $\alpha$ satisfying \eqref{alpha}), $\epsilon_{\mathrm{tol}}>0$.\\
		\hline
		$k = 0$;\\
		\textbf{Repeat}:\\
		\quad update $\bm{\theta}_{k+1}$ by \eqref{alg} using constant or backtracking stepsize $\alpha_{k+1}$;\\
		\quad $k = k+1$; \\
		\textbf{Until} stopping criterion \eqref{stopping} is satisfied.\\
		\bottomrule
	\end{tabular}
	\label{alg:prox_g}
\end{table}

\section{A specific weakly convex function and iterative firm-shrinkage method}
\label{sec:specific}
In this section, we take the weakly convex function $J$
to be the specific one defined by $F$ in \eqref{F},
in that its proximal operator has a closed form
expression  \eqref{prox_F} that is easy to compute.
We will first show a sufficient and necessary condition
for local optimality, and then discuss the proximal gradient
descent method studied in the previous section for this specific case.

\subsection{Local optimality}
With a specific function $J$ defined by $F$ in \eqref{F} we have the following conclusion on its local optimality.
\begin{theorem}\label{loc_opt}
	Suppose that function $J$ is defined by $F$ in \eqref{F}, 
	and $\beta\zeta >0.125\|{\bf X}\|^2$.
	Then $\bm{\theta}^\ast$ is a local minimum of problem \eqref{main}, 
	if and only if one of the following conditions holds for every 
	$\bm{\theta}^\ast_j$,  $j = 1,\ldots,d$.
	\begin{itemize}
		\item $\bm{\theta}^\ast_j = 0$ and $|\nabla l(\bm{\theta}^\ast)_j| < \beta$;
		\item $|\bm{\theta}^\ast_j| > \frac{1}{2\zeta}$ and
		$\nabla l(\bm{\theta}^\ast)_i = 0$.
	\end{itemize}
\end{theorem}

\begin{remark}
	If the training data points are linearly separable, i.e., there exists $\|\bm{\theta}_0\|_2 = 1$ 
	such that $ \bm{\theta}_0^{\rm T} {\bf x}^{(i)} \neq 0$ and
	\begin{align*}
	y^{(i)} = \left\{
	\begin{array}{ll}
	1, &\quad \bm{\theta}_0^{\rm T} {\bf x}^{(i)} > 0;\\
	0, &\quad \bm{\theta}_0^{\rm T} {\bf x}^{(i)} < 0,
	\end{array}
	\right.
	\end{align*}
	holds for all $i = 1,\ldots, N$,
	then 
	\[
	\nabla l(t \bm{\theta}_0) \rightarrow 0, \quad t \rightarrow + \infty,
	\]
	so we have that
	$$ \lim_{t\rightarrow +\infty} l(t\bm{\theta}_0) + \beta J(t\bm{\theta}_0) = l^\ast + \beta J^\ast $$
	is a local optimal value. For such reason in \cite{Loh2013} a constraint on a norm of $\bm{\theta}$
	is added in their optimization problem.
	However, here we note that for a given $t > 0$ the following bound holds
	$$ 0 \leq l(t\bm{\theta}_0) + \beta J(t\bm{\theta}_0) - l^\ast - \beta J^\ast 
	\leq \sum_{i=1}^N \mathrm{exp}\left(-t\left| \bm{\theta}_0^{\rm T} {\bf x}^{(i)}\right| \right) .$$
%	(In fact, when $ x < 10$, $\mathrm{exp}(-x) < 10^{-5}$.)
	The above upper bound is exponential in $t$ and decreases to $0$,
	so with a sufficiently large $t$, a point $t \bm{\theta}_0$ can give an objective value
	numerically sufficiently close to $ l^\ast + \beta J^\ast $.
\end{remark}

\begin{remark}
From Theorem \ref{loc_opt} we know that any $\bm{\theta}$ with any entry of which the absolute value is in $(0,\frac{1}{2\zeta}]$,
where $\frac{1}{2\zeta}\leq \beta/(0.25\|{\bf X}\|^2)$, 
is excluded from the solutions.
Thus, if we wish to obtain an estimated $\bm{\theta}$ with entries either have 
large enough absolute values or $0$,
then we can set such a threshold by $\zeta$.
\end{remark}

\begin{proof}
	The sufficiency can be directly obtained from Theorem \ref{thm:local_opt_suf}.
	To see this, notice that $F$ is only not differentiable at $0$, $F^\prime(0) = 1$,
	and the function
	\begin{align*}
	H(t) = \left\{
	\begin{array}{ll}
	|t|, &\quad |t| \leq \frac{1}{2\zeta} \\
	\frac{1}{4\zeta} + \zeta t^2,  &\quad |t| > \frac{1}{2\zeta} 
	\end{array}
	\right.
	\end{align*}
	satisfies the condition that both $H^{\prime\prime}_+$ and $H^{\prime\prime}_-$ are no less than $\zeta$ when and only when $|t| > 1/(2\zeta)$, where $H^\prime(t) = 2\zeta t$.
	
	From Theorem \ref{thm:local_opt_nec}, together with the assumption that
	$\beta\zeta>0.125\|{\bf X}\|^2$, we can directly have a necessary condition that every $\bm{\theta}^\ast_j$ satisfies one of the following
		\begin{itemize}
			\item $\bm{\theta}^\ast_j = 0$ and $|\nabla l(\bm{\theta}^\ast)_j| \leq \beta$;
			\item $|\bm{\theta}^\ast_j| > \frac{1}{2\zeta}$ and
			$\nabla l(\bm{\theta}^\ast)_j = 0$.
		\end{itemize}
	Thus, to prove Theorem \ref{loc_opt}, we only need to show that
	if there is a coordinate $i$ such that $\bm{\theta}^\ast_i = 0$ and $|\nabla l(\bm{\theta}^\ast)_i| = \beta$,
	then $\bm{\theta}^\ast$ is not a local optimum. To see this, we take
	\[\bm{\theta} = (\bm{\theta}^\ast_1, \ldots, \bm{\theta}^\ast_{i-1}, - t,
	\bm{\theta}^\ast_{i+1},\ldots,\bm{\theta}^\ast_d),\]
	and we will prove that for any $0<\delta<1/(2\zeta)$, there exists $t$ such that
	$0<|t| \leq \delta$ and the local optimality inequality \eqref{local_opt_ineq} does not hold. 
	According to the Lipchitz condition, we have
	\[
	l(\bm{\theta}) - l(\bm{\theta}^\ast) 
	\leq -t\nabla l(\bm{\theta}^\ast)_i + \frac{1}{8} \|{\bf X}\|^2 t^2.
	\]
	Since $\bm{\theta}^\ast_i = 0$ and $|\bm{\theta}_i|\leq 1/(2\zeta)$, we have
	\[
	G(\bm{\theta}^\ast) - G(\bm{\theta}) = -|t|.
	\]
	If $\nabla l(\bm{\theta}^\ast)_i = \beta$, then for any $t>0$
	\begin{align*}
	 l(\bm{\theta}) - l(\bm{\theta}^\ast) +\beta G(\bm{\theta}) - \beta G(\bm{\theta}^\ast) 
	+ 2\beta\zeta (\bm{\theta}^\ast - \bm{\theta})^{\rm T}  \bm{\theta}^\ast 
	- \beta\zeta\|\bm{\theta}^\ast - \bm{\theta}\|^2 
	=&l(\bm{\theta}) - l(\bm{\theta}^\ast) + \beta |t| - \beta\zeta t^2 \\
	\leq& \frac{1}{8} \|{\bf X}\|^2 t^2  - \beta\zeta t^2
	<0.
	\end{align*}
	If $\nabla l(\bm{\theta}^\ast)_i = - \beta$, then for any $t<0$ the above inequality holds.
	Therefore, we prove that the local optimality inequality \eqref{local_opt_ineq}
	cannot hold within any small neighborhood of $\bm{\theta}^\ast$, and $\bm{\theta}^\ast$
	is not a local optimum.
\end{proof}

\subsection{Iterative firm-shrinkage algorithm}

When the function $J$ is defined by $F$ in \eqref{F},
the proximal gradient method in Table \ref{alg:prox_g} discussed in section \ref{sec:solution_method}
is instantiated and can be understood as a generalization of the
iterative shrinkage-thresholding algorithm (ISTA) used
to solve $\ell_1$ regularized least square problems \cite{Teboulle2009,Hale2007}.
As the concrete proximal operator defined in \eqref{prox_F} 
has been named as the firm-shrinkage operator,
we call the method an \emph{iterative firm-shrinkage algorithm} (IFSA).

According to Theorem \ref{thm:converge}, for IFSA
if a constant or backtracking stepsize satisfying Theorem \ref{thm:converge}
is used, then we know that 
the objective function is non-increasing and convergent,
that the update $\|\bm{\theta}_k - \bm{\theta}_{k-1}\|_2$ goes to $0$,
and that any limit point of $\{\bm{\theta}_k\}$ (if there is any) is a critical point
of the objective function.

To accelerate the convergence of a proximal gradient method,
the Nesterov acceleration \cite{Nesterov1983} has been used
in ISTA \cite{Teboulle2009}, in which the convergence rate has been
accelerated from $O(1/k)$ to $O(1/k^2)$.
Such technique is also applicable to the proximal gradient method IFSA.
While the convergence analysis under such acceleration is not in the scope of this work,
we have the algorithm summarized in Table \ref{alg:acc_IFSA},

\begin{table}
	\caption{
		Iterative firm shrinkage algorithm with acceleration.
	}
	\centering
	\begin{tabular}{l}
		\toprule \textbf{Input}: initial point $\hat{\bm{\theta}}_0$, $\alpha_0<1/(2\beta\zeta)$ (or $\alpha$ satisfying \eqref{alpha}).\\
		\hline
		$k = 1$, $t_1=1$, $\bm{\theta}_1 = \hat{\bm{\theta}}_0$;\\
		\textbf{Repeat}:\\
		\quad update \\
		\quad\quad $\hat{\bm{\theta}}_{k} = \prox_{\alpha_k\beta J} (\bm{\theta}_k - \alpha_k \nabla l(\bm{\theta}_k))$ \\
		\quad \quad according to \eqref{prox_F} by constant or backtracking stepsize;\\
		\quad update $t_{k+1} = \frac{1+\sqrt{1+4t_k^2}}{2}$;\\
		\quad update $\bm{\theta}_{k+1} = \hat{\bm{\theta}}_{k} + \left(\frac{t_k-1}{t_{k+1}} \right)
		(\hat{\bm{\theta}}_k - \hat{\bm{\theta}}_{k-1})$; \\
		\quad $k = k+1$; \\
		\textbf{Until} maximum number of iterations is reached.\\
		\bottomrule
	\end{tabular}
	\label{alg:acc_IFSA}
\end{table}

\section{Numerical experiments}
\label{sec:exp}

In this section, we demonstrate numerical results 
of the weakly convex regularized sparse logistic regression
\eqref{main} with function $J$ specifically defined by $F$ in \eqref{F}.
The solving method IFSA is implemented and tested both with and without acceleration.
As a comparison,
we also show results of $\ell_1$ logistic regression,
of which there are more than one equivalent forms
and we use the one in \eqref{prob:l1}. 
There are many algorithms for $\ell_1$ logistic regression,
and we simply use a generic solver SCS interfaced by CVXPY \cite{cvxpy},
in that in such comparison we focus on replacing the $\ell_1$
norm with a weakly convex function.

\subsection{Randomly generated datasets}

\subsubsection{One example for convergence demonstration}
To begin with, for one example we show the convergence curves 
and the estimated $\bm{\theta}/\|\bm{\theta}\|_2$
of our algorithm with different constant stepsizes.
The dimensions of the data are $d = 50$, $N = 1000$, and $K=8$.
The data matrix is generated by ${\bf X} = {\bf AB}/\|{\bf AB}\|$, 
where ${\bf A}\in\reals^{50\times 45}$ and ${\bf B}\in\reals^{45\times 1000}$
are Gaussian matrices,
so that the data points are in a latent $45$-dimensional subspace.
The positions of the non-zeros of the ground truth $\bm{\theta}_0$ is uniformly
randomly generated, and the amplitudes are uniformly distributed over
$[5,15]$. The label $y$ is generated according to ${\bf 1}(\bm{\theta}_0^{\rm T} {\bf x} \geq 0)$
so that the data points are linearly separable.

We set the regularization parameter $\beta = 1.2$ and the nonconvexity
parameter $\zeta = 0.1$.
To satisfy the convergence condition in Theorem \ref{thm:converge},
we need  $\alpha<4.08$.
As a comparison, we also solve the problem with $\zeta = 0$, i.e.,
the $\ell_1$ regularized logistic regression problem \eqref{prob:l1},
by CVXPY \cite{cvxpy}.

The results without and with the Nesterov acceleration
are shown in Fig. \ref{fig:exp1} and Fig. \ref{fig:exp1_acc},
respectively.
Fig. \ref{fig:exp1} shows that, with larger stepsize (within the range),
the objective function decreases faster, and when terminated at the
given number of iterations the estimated $\bm{\theta}/\|\bm{\theta}\|_2$ becomes
closer to the ground truth. 
Fig. \ref{fig:exp1_acc}
shows that with the acceleration the objective function decreases
faster for all the tested stepsizes,
and that the estimations of $\bm{\theta}/\|\bm{\theta}\|_2$
are better than the estimation obtained from the $\ell_1$ logistic regression.

%\begin{figure}
%	\centering
%	\begin{minipage}{0.25\textwidth}
%		\centering
%		\includegraphics[width=1\linewidth]{figures/convergence.pdf}
%	\end{minipage}%
%	\begin{minipage}{0.25\textwidth}
%		\centering
%		\includegraphics[width=1\linewidth]{figures/convergence_accelerate.pdf}
%	\end{minipage}
%	\caption{Performance in an example.
%		\emph{Left:} Without acceleration.
%		\emph{Right:} With acceleration.
%		\emph{Upper:} objective value during iterations.
%		\emph{Lower:} estimated $\theta/\|\theta\|_2$ and the ground truth.}
%	\label{fig:exp1}
%\end{figure}

\begin{figure}
	\centering
		\includegraphics[width=0.8\linewidth]{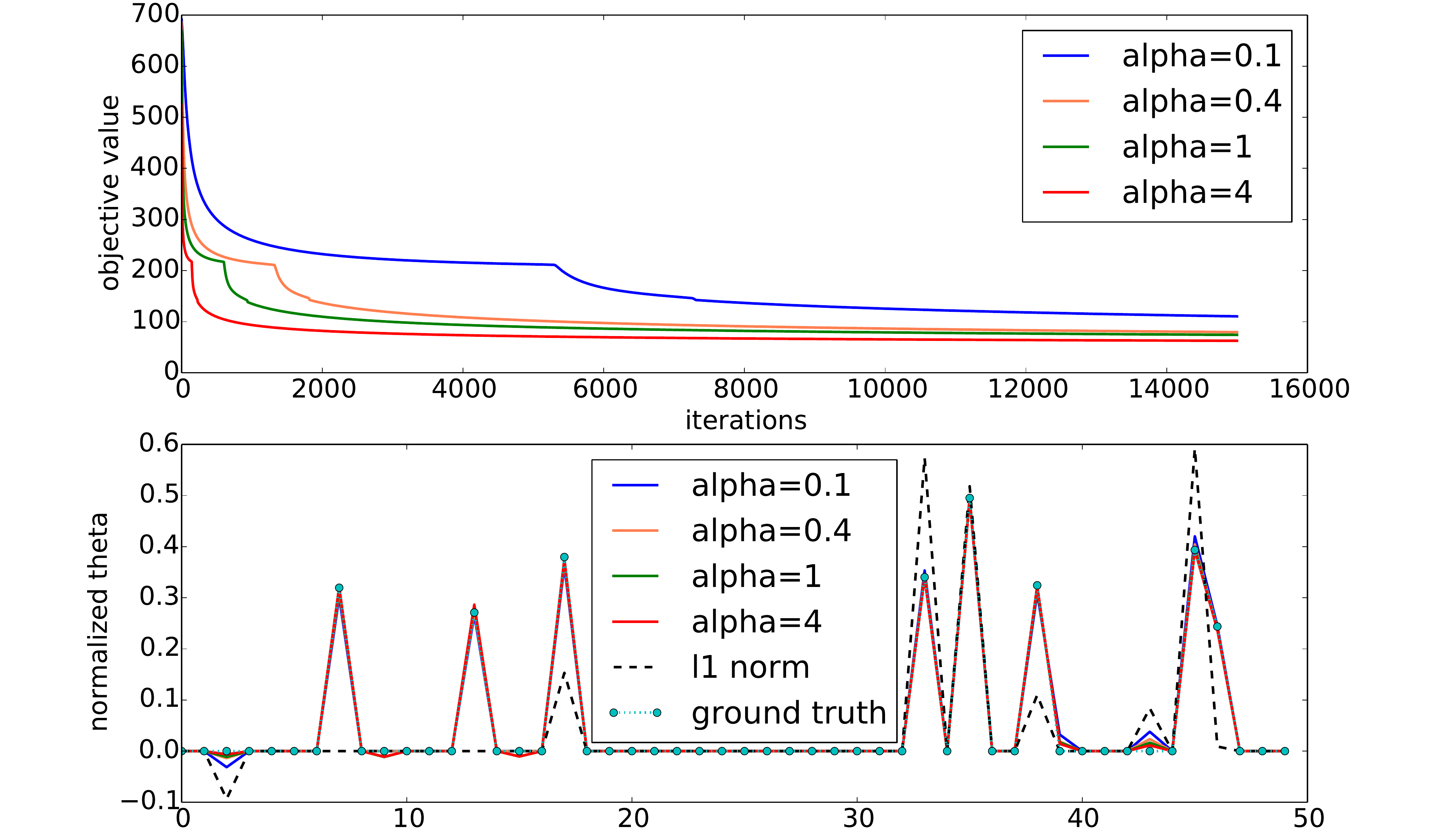}
			\caption{Performance in an example without acceleration.
				\emph{Upper:} objective value during iterations.
				\emph{Lower:} estimated $\bm{\theta}/\|\bm{\theta}\|_2$ and the ground truth.}
			\label{fig:exp1}
\end{figure}

\begin{figure}
	\centering
\includegraphics[width=0.8\linewidth]{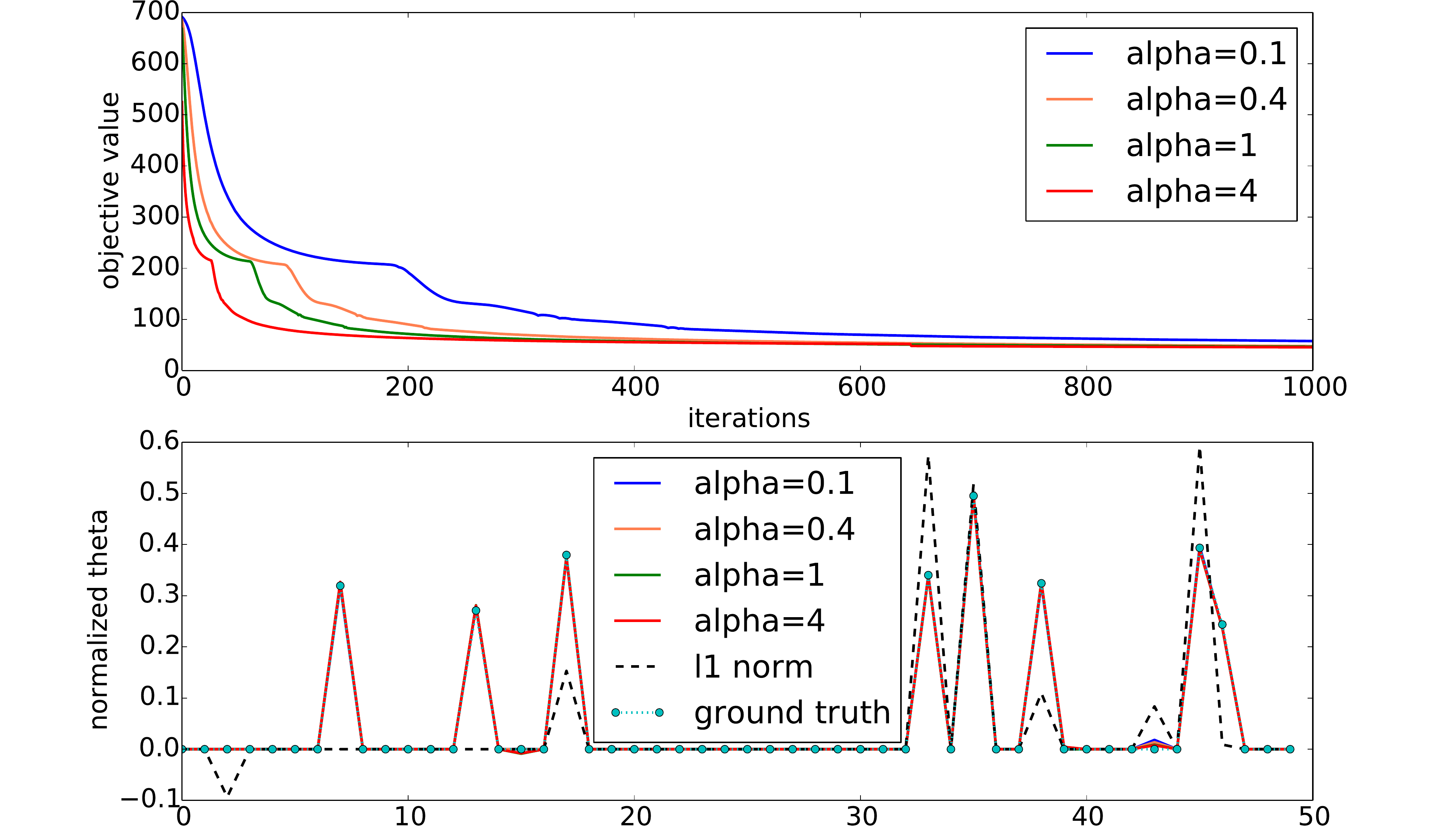}
	\caption{Performance in an example with acceleration.
		\emph{Upper:} objective value during iterations.
		\emph{Lower:} estimated $\bm{\theta}/\|\bm{\theta}\|_2$ and the ground truth.}
	\label{fig:exp1_acc}
\end{figure}

\subsubsection{Varying nonconvexity and regularization parameters}
In the second experiment, we demonstrate the performance under various 
choices of the parameters $\zeta$ and $\beta$.
The dimensions are $d = 50$, $K = 5$, and
$n = 200$.
The training data $\bf X$ is randomly generated from i.i.d.\ normal distribution,
and the ground truth $\bm{\theta}_0$ is generated by uniformly randomly choosing
$K$ nonzero elements with  i.i.d.\ normal distribution.
The step size of IFSA is chosen as $0.1$. 
The labels are generated so that the data points are linearly separable.
The results are in Fig. \ref{fig:zeta_beta}, where the horizontal axis
is the logarithm of $\zeta$, the vertical axis is the logarithm
of $\beta$, and the gray scale represents the logarithm of the test error
averaged from $10$ independently random experiments,
each of which is tested by $1000$ random test data points
which are generated in the same way as the training data points.

The results show that with a fixed value of $\beta$
from $10^{-2.8}$ to $10^{0.6}$,
as the value of $\zeta$ increases from $0$,
the test error first decreases and then increases,
and there is always a choice of $\zeta>0$ under which
the test error is smaller than the test error with $\zeta =0$
which is the $\ell_1$ logistic regression.
The results in Fig. \ref{fig:zeta_beta} verify our motivation
that weakly convex regularized logistic regression can better estimate
the sparse model than the $\ell_1$ logistic regression
and enhance test accuracy.

\begin{figure}
	\centering
	\includegraphics[width=0.7\textwidth]{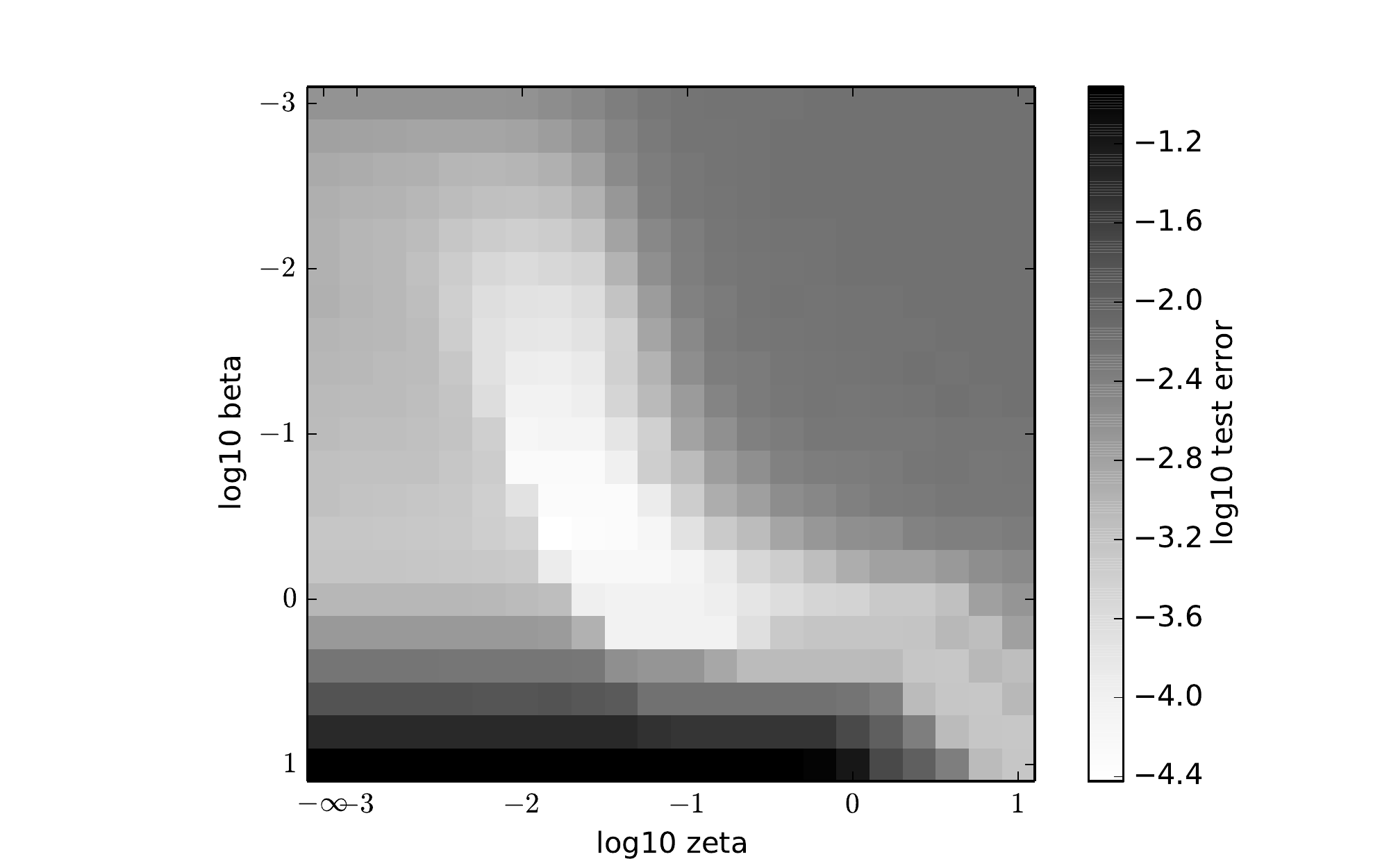}
	\caption{Logarithm of test error under various values of $\zeta$ and $\beta$.}
	\label{fig:zeta_beta}
\end{figure}

\subsubsection{Non-separable datasets}
In the above two settings the data points are linearly separable,
while in this part we will show test errors when the training data points
are not linearly separable.
To be specific, the label $y$ of a training data $\bf x$ is generated by
$y = {\bf 1}({\bf x}^{\rm T} \bm{\theta} + {\bf n} \geq 0),$
where $ {\bf n}$ is an additive noise generated from the Gaussian distribution
$\mathcal{N}(0,\epsilon^2 {\bf I})$.
The training data matrix $\bf X$, the ground truth model vector $\bm{\theta}_0$,
and the test data points
are randomly generated in the same way as the second experiment.

In the training process, under every noise level $\epsilon$,
we learned $\bm{\theta}$ under various $\beta$ from
$10^{-3}$ to $10$ and $\zeta$ from $0$ to $10$,
and we repeated it $10$ times with different random data points to take the averaged test errors for every pair of $\zeta$ and $\beta$.
For every noise level, we then took the lowest error rate obtained with $\zeta=0$ as the error rate of $\ell_1$ logistic regression 
and the lowest error rate obtained with $\zeta>0$ as the error rate
of weakly convex logistic regression. 
The results are summarized in Table \ref{unseparable}.

From the results, we can see that, 
as the noise level increases, the error rates of  both methods increase,
but  under every tested noise level the weakly convex logistic regression can achieve
lower error rate than the $\ell_1$ logistic regression.

\begin{table}
	\caption{
		Test error for non-separable data.
	}
	\centering
	\begin{tabular}{ccc}
		\toprule 
		noise level & $\ell_1$ logistic regression  & weakly convex logistic regression  \\ \hline
		$ 0.01$         & $3.27\%$  & $1.50\%$        \\ 
		$ 0.03$   & $4.94\%$     &$1.89\%$    \\
		$0.05$  & $4.21\%$  &$2.62\%$ \\
		$  0.1$  & $6.48\%$     & $5.43\%$\\ 
		$  0.3$ & $9.95\%$ & $9.27\%$ \\ 
		$  0.5$ & $20.6\%$ & $20.4\%$\\  \bottomrule
	\end{tabular}
	\label{unseparable}
\end{table}

\subsection{Real datasets}
In this part, we show experimental results of the
weakly convex logistic regression on real datasets
which have been commonly used in $\ell_1$ logistic regression,
and compare the classification error rates between these two
methods.
The first dataset is a spam email database \cite{Lichman2013}, where the number of features $57$
is far smaller than the number of data points $4601$,
of which $20\%$ are used for training.
The classification result indicates whether or not an email is a spam.
The second one is an arrhythmia dataset \cite{Lichman2013}
which has $279$ features and $452$ data points,
of which $80\%$ are used for training.
The two classes refer to the normal and arrhythmia,
and missing values in the features are filled with zeros.
The third one is a gene database from tumor and normal tissues \cite{Alon08061999}, 
where the number of features $2000$ is far larger than the number of data points
$62$, of which $40\%$ are used for training.
The classification result is whether or not it is a tumor tissue.
The test and training data points are randomly separated.

In the experiments,
we first run $\ell_1$ logistic regression with various $\beta$ on the training data
and use cross validation on the test data to find the best value of $\beta$
and the corresponding error rate.
Then we run the IFSA for weakly convex logistic regression
under the best $\beta$ with various $\zeta$,
and still use cross validation to get the best $\zeta$ and its error rate.

Results in Table~\ref{real_datasets} show that,
for the first dataset, where the number of features is far smaller than
the number of training data, the weakly convex logistic regression
has a little improvement over the $\ell_1$ regularized logistic regression.
For the second and the third datasets,
where the number of training data points is inadequate compared to the number of features,
the improvement achieved by weakly convex logistic regression is more significant.

\begin{table}
		\caption{
			Results on real datasets.
		}
		\centering
		\begin{tabular}{cccc}
			\toprule 
			              										& Spambase   & Arrhythmia & Colon   \\ \hline
		 number of training samples     & $921$           & $361$           & $25$     \\ 
		 number of features                    & $57$             & $279$         &$2000$ \\
		 number of test samples           & $3680$        & $91$         &$37$    \\
		$\ell_1$ logistic regression best $\beta$ & $0.0062$  & $0.01$ & $0.0046$\\ 
		$\ell_1$ logistic regression error rate       & $8.23\%$   & $24.18\%$ & $28\%$ \\ 
	    weakly convex best $\zeta$                      & $0.006$ & $0.0055$& $0.01$\\ 
		weakly convex error rate                            & $7.96\%$ & $18.68\%$ & $24\%$\\ \bottomrule
		\end{tabular}
	\label{real_datasets}
\end{table}

\section{Conclusion and future work}
In this work we study weakly convex regularized sparse logistic regression.
For a class of weakly convex sparsity inducing functions,
we first prove that the optimization problem with such functions as regularizers
is in general nonconvex, and then we study its local optimality conditions,
as well as the choice of the regularization parameter to exclude a trivial solution.
Even though the general problem is nonconvex,
a solution method based on the proximal gradient descent
is devised with theoretical convergence analysis.
Then the general framework is applied to a specific weakly convex function, and a necessary and
sufficient local optimality condition is unveiled.
The solution method for this specific case named iterative firm-shrinkage algorithm is implemented.
Its effectiveness is demonstrated in numerical experiments
by both randomly generated data and real datasets.

There can be several directions to extend this work, 
such as using only parts of the data in every iteration
by applying stochastic proximal gradient method.
More generally, weakly convex regularization could be used in
other  machine learning problems to fit sparse models.

\bibliographystyle{ieeetr}
\bibliography{refs1}

\begin{thebibliography}{10}

\bibitem{Genkin2006}
A.~Genkin, D.~D. Lewis, and D.~Madigan, ``Large-scale bayesian logistic
  regression for text categorization,'' {\em Technometrics}, 2006.

\bibitem{Zhu2004}
J.~Zhu and T.~Hastie, ``Classification of gene microarrays by penalized
  logistic regression,'' {\em Biostatistics}, vol.~5, no.~3, pp.~427--43, 2004.

\bibitem{Cawley2006}
G.~C. Cawley and N.~L. Talbot, ``Gene selection in cancer classification using
  sparse logistic regression with bayesian regularization,'' {\em
  Bioinformatics}, vol.~22, pp.~2348--2355, 9 2006.

\bibitem{Cronin2002}
M.~T.~D. Cronin, A.~O. Aptula, J.~C. Dearden, J.~C. Duffy, T.~I. Netzeva,
  H.~Patel, P.~H. Rowe, T.~W. Schultz, A.~P. Worth, and K.~Voutzoulidis,
  ``Structure-based classification of antibacterial activity,'' {\em Journal of
  Chemical Information and Computer Sciences}, vol.~42, no.~4, p.~869, 2002.

\bibitem{Murray2011}
R.~F. Murray, ``Classification images: A review.,'' {\em Journal of Vision},
  vol.~11, no.~5, pp.~74--76, 2011.

\bibitem{Ciocca2015}
G.~Ciocca, C.~Cusano, and R.~Schettini, ``Image orientation detection using
  {LBP}-based features and logistic regression,'' {\em Multimedia Tools and
  Applications}, vol.~74, no.~9, pp.~3013--3034, 2015.

\bibitem{hastie01}
T.~Hastie, R.~Tibshirani, and J.~Friedman, {\em The Elements of Statistical
  Learning}.
\newblock Springer Series in Statistics, Springer New York Inc., 2001.

\bibitem{Chaloner1989}
K.~Chaloner and K.~Larntz, ``Optimal bayesian design applied to logistic
  regression experiments,'' {\em Journal of Statistical Planning and
  Inference}, vol.~21, no.~2, pp.~191--208, 1989.

\bibitem{Komarek2004}
P.~Komarek, {\em Logistic Regression for Data Mining and High-dimensional
  Classification}.
\newblock PhD thesis, Pittsburgh, PA, USA, 2004.
\newblock AAI3121277.

\bibitem{Minka2007}
T.~P. Minka, ``A comparison of numerical optimizers for logistic regression,''
  {\em J.am.chem.soc}, vol.~125, no.~6, pp.~1660--1668, 2007.

\bibitem{Roth2002}
V.~Roth, ``The generalized lasso: a wrapper approach to,'' {\em IEEE Trans
  Neural Netw}, vol.~15, no.~1, pp.~16 -- 28, 2002.

\bibitem{Shevade2003}
S.~K. Shevade and S.~S. Keerthi, ``A simple and efficient algorithm for gene
  selection using sparse logistic regression,'' {\em Bioinformatics}, vol.~19,
  no.~17, pp.~2246--2253, 2003.

\bibitem{Lee2006}
S.~Lee, H.~Lee, P.~Abbeel, and A.~Y. Ng, ``Efficient l1 regularized logistic
  regression,'' in {\em AAAI}, 2006.

\bibitem{Koh2007}
K.~Koh, S.~Kim, and S.~Boyd, ``An interior-point method for large-scale
  l1-regularized logistic regression,'' {\em Journal of Machine Learning
  Research}, vol.~8, no.~4, pp.~1519--1555, 2007.

\bibitem{plan2013}
Y.~Plan and R.~Vershynin, ``Robust 1-bit compressed sensing and sparse logistic
  regression: A convex programming approach,'' {\em IEEE Transactions on
  Information Theory}, vol.~59, no.~1, pp.~482--494, 2013.

\bibitem{Loh2013}
P.~Loh and M.~J. Wainwright, ``Regularized {M}-estimators with nonconvexity:
  Statistical and algorithmic theory for local optima,'' in {\em Advances in
  Neural Information Processing Systems 26}, pp.~476--484, 2013.

\bibitem{Yang2016}
L.~Yang and Y.~Qian, ``A sparse logistic regression framework by difference of
  convex functions programming,'' {\em Applied Intelligence}, vol.~45,
  pp.~241--254, Sep 2016.

\bibitem{Chartrand2007}
R.~Chartrand, ``Exact reconstruction of sparse signals via nonconvex
  minimization,'' {\em IEEE Signal Processing Letters}, vol.~14, no.~10,
  pp.~707--710, 2007.

\bibitem{Chen2014}
L.~Chen and Y.~Gu, ``The convergence guarantees of a non-convex approach for
  sparse recovery,'' {\em IEEE Transactions on Signal Processing}, vol.~62,
  no.~15, pp.~3754--3767, 2014.

\bibitem{zhu2015}
R.~Zhu and Q.~Gu, ``Towards a lower sample complexity for robust one-bit
  compressed sensing,'' in {\em Proceedings of the 32nd International
  Conference on Machine Learning (ICML15)} (D.~Blei and F.~Bach, eds.),
  pp.~739--747, 2015.

\bibitem{Shen2016}
X.~Shen, L.~Chen, Y.~Gu, and H.~C. So, ``Square-root lasso with nonconvex
  regularization: An admm approach,'' {\em IEEE Signal Processing Letters},
  vol.~23, no.~7, pp.~934--938, 2016.

\bibitem{Tibshirani1996}
R.~Tibshirani, ``Regression shrinkage and selection via the lasso,'' {\em
  Journal of the Royal Statistical Society. Series B (Methodological)},
  vol.~58, no.~1, pp.~267--288, 1996.

\bibitem{Perkins03}
S.~Perkins and J.~Theiler, ``Online feature selection using grafting,'' in {\em
  In International Conference on Machine Learning}, pp.~592--599, ACM Press,
  2003.

\bibitem{LeThi2008}
H.~A. Le~Thi, H.~M. Le, V.~V. Nguyen, and T.~Pham~Dinh, ``A {DC} programming
  approach for feature selection in support vector machines learning,'' {\em
  Advances in Data Analysis and Classification}, vol.~2, pp.~259--278, Dec
  2008.

\bibitem{Cheng2013}
S.~O. Cheng and H.~A. Le~Thi, ``Learning sparse classifiers with difference of
  convex functions algorithms,'' {\em Optimization Methods and Software},
  vol.~28, no.~4, pp.~830--854, 2013.

\bibitem{Boufounos2008}
P.~T. Boufounos and R.~G. Baraniuk, ``1-bit compressive sensing,'' in {\em 2008
  42nd Annual Conference on Information Sciences and Systems}, pp.~16--21,
  March 2008.

\bibitem{Donoho2006}
D.~Donoho, ``Compressed sensing,'' {\em IEEE Transactions on Information
  Theory}, vol.~52, no.~4, pp.~1289--1306, 2006.

\bibitem{Chen2014c}
L.~Chen and Y.~Gu, ``The convergence guarantees of a non-convex approach for
  sparse recovery using regularized least squares,'' in {\em 2014 IEEE
  International Conference on Acoustics, Speech and Signal Processing
  (ICASSP)}, pp.~3350--3354, 2014.

\bibitem{Chen2015}
L.~Chen and Y.~Gu, ``Fast sparse recovery via non-convex optimization,'' in
  {\em 2015 IEEE Global Conference on Signal and Information Processing
  (GlobalSIP)}, 2015.

\bibitem{Vial1983}
J.~Vial, ``Strong and weak convexity of sets and functions,'' {\em Mathematics
  of Operations Research}, vol.~8, no.~2, pp.~231--259, 1983.

\bibitem{Zhang2010}
C.~Zhang, ``Nearly unbiased variable selection under minimax concave penalty,''
  {\em The Annals of Statistics}, vol.~38, no.~2, pp.~894--942, 2010.

\bibitem{Gao1997}
H.~Gao and A.~G. Bruce, ``Waveshrink with firm shrinkage,'' {\em Statistica
  Sinica}, vol.~7, no.~4, pp.~855--874, 1997.

\bibitem{An2005}
A.~Le~Thi~Hoai and T.~Pham~Dinh, ``The {DC} (difference of convex functions)
  programming and {DCA} revisited with {DC} models of real world nonconvex
  optimization problems,'' {\em Annals of Operations Research}, vol.~133,
  no.~1-4, pp.~23--46, 2005.

\bibitem{Teboulle2009}
A.~Beck and M.~Teboulle, ``A fast iterative shrinkage-thresholding algorithm
  for linear inverse problems,'' {\em SIAM Journal on Imaging Sciences},
  vol.~2, no.~1, pp.~183--202, 2009.

\bibitem{beck2009}
A.~Beck and M.~Teboulle, {\em Gradient-based algorithms with applications to
  signal-recovery problems}, pp.~42--88.
\newblock Cambridge University Press, 2009.

\bibitem{Hale2007}
E.~T. Hale, W.~Yin, and Y.~Zhang, ``A fixed-point continuation method for
  l1-regularized minimization with applications to compressed sensing,'' {\em
  CAAM Technical report TR07-07}, 2007.

\bibitem{Nesterov1983}
Y.~E. Nesterov, ``A method for solving the convex programming problem with
  convergence rate ${O}(1/k^{2})$,'' pp.~372--376, 1983.

\bibitem{cvxpy}
S.~Diamond and S.~Boyd, ``{CVXPY}: A {P}ython-embedded modeling language for
  convex optimization,'' {\em Journal of Machine Learning Research}, vol.~17,
  no.~83, pp.~1--5, 2016.

\bibitem{Lichman2013}
M.~Lichman, ``{UCI} machine learning repository,'' 2013.

\bibitem{Alon08061999}
U.~Alon, N.~Barkai, D.~A. Notterman, K.~Gish, S.~Ybarra, D.~Mack, and A.~J.
  Levine, ``Broad patterns of gene expression revealed by clustering analysis
  of tumor and normal colon tissues probed by oligonucleotide arrays,'' {\em
  Proceedings of the National Academy of Sciences}, vol.~96, no.~12,
  pp.~6745--6750, 1999.

\end{thebibliography}

%\end{CJK*}
\end{document}